\documentclass{article}

\PassOptionsToPackage{numbers, sort,compress}{natbib}

\usepackage[final]{neurips_2023}

\usepackage[utf8]{inputenc} %
\usepackage[T1]{fontenc}    %
\usepackage{hyperref}       %
\usepackage{url}            %
\usepackage{booktabs}       %
\usepackage{amsfonts}       %
\usepackage{nicefrac}       %
\usepackage{microtype}      %
\usepackage{xcolor}         %
\usepackage{amsmath,amsthm}
\usepackage{dsfont}
\usepackage{enumitem}
\usepackage{mathtools}

\newcommand{\dd}{\mathrm{d}}
\newcommand{\kl}{\mathrm{KL}}
\newcommand{\iidsim}{\overset{\text{iid}}{\sim}}

\DeclareMathOperator*{\argmin}{arg\,min}

\newtheorem{theorem}{Theorem}
\newtheorem{definition}{Definition}
\newtheorem*{assumption}{Assumptions}

\title{A PAC-Bayesian Perspective on the \\ Interpolating Information Criterion}

\author{%
  Liam Hodgkinson \\
  School of Mathematics and Statistics\\
  University of Melbourne\\
  \texttt{lhodgkinson@unimelb.edu.au} \\
  \And
  Chris van der Heide \\
  Department of Electrical and Electronic Engineering \\
  University of Melbourne\\
  \texttt{chris.vdh@gmail.com}\\
  \And
  Robert Salomone \\
  Centre for Data Science \\
  Queensland University of Technology \\
  \texttt{robert.salomone@qut.edu.au} \\
  \And
  Fred Roosta \\
  CIRES and School of Mathematics and Physics \\
  University of Queensland \\
  \texttt{fred.roosta@uq.edu.au}\\
  \And
  Michael W. Mahoney \\
  ICSI, LBNL, and Department of Statistics \\
  University of California, Berkeley \\
  \texttt{mmahoney@stat.berkeley.edu}
}

\begin{document}

\maketitle

\begin{abstract}
Deep learning is renowned for its theory--practice gap, 
whereby principled theory typically fails to provide much beneficial guidance for implementation in practice. 
This has been highlighted recently by the benign overfitting phenomenon: when neural networks become sufficiently large to interpolate the dataset perfectly, model performance appears to improve with increasing model size, in apparent contradiction with the well-known bias--variance tradeoff. 
While such phenomena have proven challenging to theoretically study for general models, the recently proposed Interpolating Information Criterion (IIC) provides a valuable theoretical framework to examine performance for overparameterized models. 
Using the IIC, a PAC-Bayes bound is obtained for a general class of models, characterizing factors which influence generalization performance in the interpolating regime.
From the provided bound, we quantify how the test error for overparameterized models achieving effectively zero training error depends on the quality of the implicit regularization imposed by e.g. the combination of model, optimizer, and parameter-initialization scheme; the spectrum of the empirical neural tangent kernel; curvature of the loss landscape; and noise present in the data.
\end{abstract}

\section{Introduction}

A prominent curiosity in modern machine learning is the occurrence of strong generalization performance, even in the \emph{overparameterized} setting where the number of parameters exceeds the size of the training set and models can \emph{interpolate} even noisy data \cite{belkin2019reconciling,zhang2021understanding}. 
This is at odds with classical theoretical arguments in line with the bias--variance tradeoff, as interpolators are typically thought to correspond to high-variance estimators in the presence of data noise, and therefore should perform poorly \cite[\S2.9]{hastie2009elements}. 
Such observations have sparked renewed interest in \emph{interpolating estimators} and the occurrence of \emph{benign overfitting} \cite{bartlett2020benign,frei2023benign,tsigler2023benign}. 
One of the more celebrated realizations of benign overfitting is the \emph{double descent} curve particularly pronounced in linear regression \cite{belkin2019reconciling,hastie2022surprises,derezinski2020exact,liao2020random}, where %
model mean-squared error \emph{decreases} monotonically in the overparameterized regime. 
This surprising behaviour arises due to the \emph{implicit regularization} present in the choice of estimator \cite{neyshabur2017implicit}.
However, rigorous theoretical examination of these phenomena beyond the linear setting remains a significant challenge. 
For example, analogous curves can become arbitrarily complicated in the kernel regression setting~\cite{chen2021multiple,liang2020just,liu2021kernel}.

The problem of \emph{model selection} becomes exacerbated in the overparameterized setting: how do we compare between interpolators?
Classically, model selection is conducted using an \emph{information criterion}, the most prominent of which are the AIC and BIC~\cite{konishi2008information}, although these all break down for overparameterized models.
One recent approach to model selection in the \emph{general} overparameterized setting is presented in \cite{hodgkinson2023interpolating} with the \emph{Interpolating Information Criterion} (IIC).\footnote{This is related to but substantially more general than previous work in the linear/kernel setting~\cite{hodgkinson2022monotonicity}.}
Adopting a Bayesian setup, performance for the IIC is measured in terms of the \emph{marginal likelihood}. 
Similar to \cite{hanin2023bayesian}, the interpolating regime is examined through the \emph{cold posterior} scenario, where the temperature of the likelihood is decreased to concentrate posterior mass onto the zero-loss set of parameters.
The IIC itself relies on a novel (and broadly applicable) principle of \emph{Bayesian duality} \cite{hodgkinson2023interpolating}: for any overparameterized model, there exists a corresponding underparameterized model with the same marginal likelihood. 
Conveniently, this corresponding underparameterized model is often amenable to asymptotic approximations via Laplace's method, resulting in a tractable form of the marginal likelihood, even for complex models. 

The IIC is theoretically interesting, %
but its utility may not be immediately obvious.
While the marginal likelihood is a standard in Bayesian statistics, it is not often a metric of choice for machine learning practitioners.
Several deficiencies in the marginal likelihood have been raised as detrimental to accurate examination of model quality \cite{lotfi2022bayesian}.
Instead, a more popular framework for assessing model performance using Bayesian ideas is that of \emph{PAC-Bayes bounds} \cite{alquier2021user}. 
These bounds on the true risk are often more straightforward to interpret in practice, and provide the tightest estimates of the test error to date \cite{dziugaite2017computing,lotfi2022pac}. %
However, PAC-Bayes bounds are often limited by their requirement of a tractable choice of prior. Hence, previous bounds have only been capable of revealing coarse attributes (such as norm-based metrics \cite{neyshabur2017exploring,neyshabur2017pac}) linked to generalization through specific choices of the prior \cite{jiang2019fantastic}. 

Using techniques from the derivation of the IIC in \cite{hodgkinson2023interpolating}, we construct a PAC-Bayes bound that holds for a very wide class of models \emph{and} priors in the general overparameterized setting. 
In doing so, we provide a precise and \emph{complete} characterization of how model performance for interpolators depends on the quality of the implicit regularization, the sharpness of the model about the estimator, the curvature of the zero-loss region in the loss landscape, and the noise of the data. 
While earlier attempts have been made to develop PAC-Bayesian generalization bounds in the cold posterior setting \cite{pitas2022cold}, these are again limited by strong simplifying assumptions. 
In constrast, our PAC-Bayes bound holds for a general class of interpolators, with minimal assumptions on the regularity of the model.

\section{Interpolating Regime}

Parameter estimators for regression problems are typically minimizers of an empirical risk $L_n$:
\[
\theta^\star \in \mathcal{M} \coloneqq \argmin_{\theta \in \Theta} L_n(\theta),\quad \text{where} \quad L_n(\theta) = \frac1n\sum_{i=1}^n \ell(f(x_i,\theta),y_i),
\]
where $x_1,\dots,x_n \in \mathcal{X}$ are inputs, $y_1,\dots,y_n \in \mathbb{R}^m$ are the corresponding outputs, and $\Theta \subset \mathbb{R}^d$ is the parameter space. For example, in deep learning, $f:\mathcal{X} \times \mathbb{R}^d \to \mathbb{R}^m$ prescribes a (nonlinear) neural network architecture with $d$ weights and $m$ outputs over the input space $\mathcal{X}$. For simplicity, assume $\Theta = \mathbb{R}^d$ and restrict our attention to the mean-squared loss $\ell(z,y) = \|z - y\|^2$, although more general loss functions can also be considered. 

When the number of parameters $d$ exceeds the size of the dataset $mn$ and the model can \emph{interpolate} the data exactly, $\mathcal{M}$ is often uncountable. So, which $\theta \in \mathcal{M}$ should be chosen? A convenient approach is to select an estimator within $\mathcal{M}$ that is the solution to a constrained optimization problem involving a regularizer $R$ \cite{belkin2021fit}. 
\begin{definition}
\label{def:Interpolator}
An interpolator is an estimator of the form
$\theta^\star = \argmin_{\theta \in \Theta} R(\theta)$ subject to $f(x_i,\theta) = y_i$ for all $i=1,\dots,n$, where $R:\mathbb{R}^d \to \mathbb{R}$. 
\end{definition}
We assume that $R$ has both a unique minimizer over $\mathbb{R}^d$ \emph{and} a unique minimizer over $\mathcal{M}$. 
As observed in \cite{hodgkinson2023interpolating}, interpolators as prescribed in Definition \ref{def:Interpolator} arise naturally in the Bayesian context, as we now demonstrate. 
To start, consider the usual Bayesian posterior $\rho_\gamma(\theta) \propto \exp(-\tfrac1\gamma L_n(\theta)) \pi(\theta)$  formed from the Gibbs likelihood with temperature $\gamma$ and prior $\pi$. 
In the limit as $\gamma \to 0^+$, $\rho_\gamma$ will concentrate around regions where $L_n(\theta)$ is minimized, namely $\mathcal{M}$\footnote{This limiting behaviour in the posterior was quantified and investigated in \cite{de2021quantitative}.}.
This is the \emph{cold posterior} setting, which has surprisingly been observed to enhance predictive performance \cite{wenzel2020good}. In this setting, the role of $\pi$ in prescribing mass to estimators on $\mathcal{M}$ is enhanced. %
If we now choose $\pi(\theta) \propto \exp(-\frac1\tau R(\theta))$ to be the Gibbs measure corresponding to $R$ with temperature $\tau$, then regions with high probability under $\pi$ correspond to smaller values of $R$. 
In this way, $R$ acts as a regularizer over the set of interpolators. 
By taking $\tau \to 0^+$, the cold posterior concentrates on the minimizer $\theta^\star$ of $R$ on $\mathcal{M}$. 

The order of the limits ($\gamma \to 0^+$, and then $\tau \to 0^+$) is of great importance, as the limiting behaviour of $\rho_\gamma$ varies greatly depending on the relative rates with which these two limits are taken \cite{fulks1951generalization}. 
For example, if $\gamma$ and $\tau$ reduce at the same rate, then $\rho_\gamma$ concentrates around a \emph{maximum a posteriori} (MAP) estimator, and if $\tau \to 0^+$ first, then $\rho_\gamma$ concentrates on the unique minimizer of $R$.

In modern machine learning, the limit $\gamma \to 0^+$ corresponds to the procedure of optimizing the model to zero empirical risk. These asymptotics can be equally viable for models which achieve \emph{sufficiently small} training error. If $\theta^\star$ represents the trained weights, then $R$ is the implicit regularizer of the model, comprising all the factors (choice and hyperparameters of the optimizer, initialization, etc.) which dictate the particular solution reached at the end of training. By examining the true risk over the posterior $\rho_\gamma$ under the limits $\gamma \to 0^+$, and then $\tau \to 0^+$, we reveal a localized estimate of test error for large neural networks at the end of training. 
Within $\mathcal{M}$, $R$ is the primary measurement distinguishing between estimators, and plays an analogous role to the ``risk'' in the bound to~follow.

\section{PAC-Bayes Bounds}

Performance in machine learning is typically analyzed through the \emph{true risk function}, $L(\theta)$. 
Assuming that each $(x_i,y_i)$ is an iid realization from a distribution $\mathcal{D}$, we let $L(\theta) = \mathbb{E}_{(x,y)\sim \mathcal{D}} \ell(f(x,\theta),y)$. 
The difference between $L_n(\theta)$ and $L(\theta)$ is referred to as \emph{generalization error}.
A small value for the error for well-trained models is typically thought to be indicative of good real-world performance, and hence high model quality. 
A Bayesian analogue of the PAC framework, called \emph{PAC-Bayes theory}, was first introduced by McAllester~\cite{mcallester1999some} and has become recognized as a promising approach for potentially-practical non-vacuous bounds on the generalization error \cite{dziugaite2017computing}. 
Following \cite{begin2016pac}, and letting $\kl(\cdot\Vert\cdot)$ denote the Kullback--Leibler divergence, the Donsker--Varadhan change of measure theorem \cite[Lemma 3]{begin2016pac} applied to two probability measures, $\rho$ and $\pi$, states that
\[
\mathbb{E}_{\theta\sim\rho}\phi(\theta)\leq\kl(\rho\Vert\pi)+\log\mathbb{E}_{\theta\sim\pi}e^{\phi(\theta)},\qquad\mbox{for any }\phi:\Theta\to\mathbb{R}.
\]
If $\phi$ is $n$-times the generalization error, then using Markov's inequality, with probability at least $1-\delta$ over the choice of $(x_i,y_i) \iidsim \mathcal{D}$, it holds that
\begin{equation}
\label{eq:PACBayesCore}
\mathbb{E}_{\theta\sim\rho}L(\theta)\leq\mathbb{E}_{\theta\sim\rho}L_{n}(\theta)+\frac{1}{n}\left[\kl(\rho\Vert\pi)+\log\mathbb{E}_{(x,y)\sim\mathcal{D}}\zeta_{n}+\log(1/\delta)\right],
\end{equation}
where $\zeta_n = \mathbb{E}_{\theta\sim\pi} e^{n (L(\theta)-L_n(\theta))}$ encodes dispersion in the loss under the prior due to noise in the data. While (\ref{eq:PACBayesCore}) holds for arbitrary measures $\rho$ and $\pi$, the bound is tightest when $\rho$ is the posterior $\rho_\gamma$ with $\gamma=1/n$ \cite[\S2.1]{alquier2021user}.
Equation (\ref{eq:PACBayesCore}) is the core PAC-Bayes bound: to minimize the true risk, one should optimize over the empirical risk, and choose a prior that is as close to the posterior as possible. 
Effective choices of priors have resulted in non-vacuous generalization bounds, even for moderately large-scale neural networks \cite{dziugaite2017computing,lotfi2022pac}. 
However, the Kullback--Leibler term is often intractable for arbitrary priors, and so $\pi$ is typically chosen to render the right-hand side explicitly computable. There are two major issues with this: (i) the true role of the implicit regularization---believed to be \emph{critical} \cite{neyshabur2017implicit,zhang2021understanding}---remains opaque, as only a simplified version of this regularization can be examined; and (ii) the bound can only be as tight as one can approximate the optimal choice of prior. %

An alternative approach is to trade strict upper bounds for \emph{asymptotics} in the interpolating regime using the techniques of \cite{hodgkinson2023interpolating}, along with the observations of \cite{germain2016pac}. By doing so, a tractable PAC-Bayes bound is obtained for almost any choice of prior, opening the door to a more precise theoretical understanding of regularization, and potentially tighter bounds.
The resulting PAC-Bayes bound depends on the performance of the interpolator $\theta^\star$ under the regularizer $R$, and it makes explicit the dependence of model performance on three key factors:
\begin{itemize}[leftmargin=*]
\item \textbf{Sharpness:} $S = \log \det (DF(\theta^\star) DF(\theta^\star)^\top)$; where $DF(\theta)$ is the $nm\times d$ Jacobian of $F$ with rows $(\nabla_\theta F(x_i,\theta))_{i=1}^n$. Sharpness measures are well-known to (sometimes~\cite{yao2020pyhessian}) correlate with performance \cite{hochreiter1997flat,keskar2016large}. 
Note that $S$ is the log-determinant of the \emph{empirical neural tangent kernel} (NTK) \cite{jacot2018neural,novak2022fast}.
\item \textbf{Dispersion:} $P=2n^{-1} \log \mathbb{E}_{(x,y)\sim\mathcal{D}}e^{n(L(\theta_{0})-L_{n}(\theta_{0}))}$; where $\theta_0 \in \mathbb{R}^d$ is the assumed global minimizer of the regularizer $R$. 
Note that if $\ell(f(x,\theta_0),y)$ is normally distributed, then $P$ is its variance. 
However, as $P$ does not depend on the posterior, and so plays a limited role in our bound.
\item \textbf{Curvature:} $K = \log \det_+ \nabla_{\mathcal{M}}^2 R(\theta^\star) - \log \det \nabla^2 R(\theta_0)$; where $\det_+$ is the pseudo-determinant (product of all non-zero eigenvalues) and $\nabla_{\mathcal{M}}^2$ is the manifold Hessian over $\mathcal{M}$ \cite[\S5.5]{absil2008optimization}. The manifold Hessian over $\mathcal{M}$ is well-defined according to the Implicit Function Theorem, which asserts that $\mathcal{M}$ is a submanifold of dimension $d - mn$ if $DF$ is continuous and full rank on $\mathbb{R}^d$.
\end{itemize}
The interpolating information criterion as presented in \cite{hodgkinson2023interpolating} is given in terms of these factors as
\begin{equation}
\label{eq:IICOrig}
\text{IIC} = \log [R(\theta^\star)-R(\theta_0)] + \frac{S + K}{mn} - \log(mn),
\end{equation}
where a smaller IIC is indicative of better model performance. For more details on the nature of these factors, see \cite[\S5.1]{hodgkinson2023interpolating}.
Following \cite{hodgkinson2023interpolating}, our analysis operates under the following conditions. Of these, (F) is perhaps the most unusual condition, but is necessary to ensure that $\mathbb{E}_{\theta \sim \rho_\gamma} L(\theta) \neq 0$.
\begin{assumption}
Assume the following conditions:\vspace{-.2cm}
\begin{enumerate}[label=(\Alph*),leftmargin=1cm]
\itemsep0em 
\item $F$ and $R$ are $\mathcal{C}^{\infty}$-smooth on $\mathbb{R}^d$
\item $\mathcal{M}$ is non-empty, and $DF(\theta)$ is full rank for all $\theta \in \mathbb{R}^d$
\item the function $\theta \mapsto \pi(\theta) \det(DF(\theta)DF(\theta)^\top)^{-1/2}$ is integrable over $\mathbb{R}^d$
\item the manifold Hessian $\nabla_{\mathcal{M}}^2 R(\theta^\star)$ is non-singular
\item $R(\theta) \leq M\|\theta\|^p$ for all $\theta \in \mathbb{R}^d$ for some $M,p > 0$
\item the normalizing constant for $\rho_\gamma$ is bounded as $\gamma \to 0^+$
\item $R(\theta^\star)$ is uniformly bounded for any $n=1,2,\dots$
\end{enumerate}
\end{assumption}

With these assumptions, we can establish the following theorem, our main result, the proof of which can be found in Appendix~\ref{sxn:appendix}.
\begin{theorem}[PAC-Bayes Bound for Interpolators]
\label{thm:PAC}
Consider the cold posterior $\rho_\gamma$ with prior $\pi(\theta) \propto \exp(-\frac1\tau R(\theta))$ under the choice of temperature $\tau = \frac{2}{mn} [R(\theta^\star) - R(\theta_0)]$. Then for any $0 < \delta < 1$, with probability at least $1-\delta$, as $\gamma \to 0^+$, and then $n \to \infty$,
\begin{align}
2\mathbb{E}_{\theta\sim\rho_\gamma}L(\theta)	&\leq m \log(R(\theta^{\ast})-R(\theta_{0}))+\frac{1}{n}S+\frac{1}{n}K+P \nonumber \\
&+m\left(1-\log\frac{mn}{2\pi}\right)+\frac{1}{n}\log\left(\delta^{-2}\right)+\mathcal{O}(n^{-2})+\mathcal{O}(\gamma)  . \label{eq:PACBound}
\end{align}
\end{theorem}
In Theorem \ref{thm:PAC}, the temperature $\tau$ is chosen so as to minimize the bound, excluding higher-order terms. The bound (\ref{eq:PACBound}) has a similar interpretation to the IIC in \cite{hodgkinson2023interpolating} and so most of the discussion there is also relevant here. Indeed, in terms of the IIC in (\ref{eq:IICOrig}), the bound (\ref{eq:PACBound}) becomes
\[
2 \mathbb{E}_{\theta \sim \rho_\gamma} L(\theta) \leq m \cdot (1 + \log(2\pi) + \text{IIC}) + P + n^{-1} \log(\delta^{-2}) + \mathcal{O}(n^{-2}) + \mathcal{O}(\gamma).
\]

\section{Discussion and Conclusions}

A PAC-Bayesian bound is presented in Theorem \ref{thm:PAC} for interpolators in the overparameterized regime, using the results of the IIC \cite{hodgkinson2023interpolating}. Our bound is quite general, imposing few restrictions on the model and the form of its implicit regularization. This is particularly advantageous in the setting of deep learning, where the precise nature of the model and the training process is often complex. 

Drawing particular attention to the factor $S$, recall sharpness of the loss landscape is typically quantified in terms of the Hessian of the loss \cite{yao2018hessian}. Multiple examinations have reported limitations to sharpness metrics computed involving the Hessian \cite{SHS2023,ACMHF2023}. One possibility for this deficiency is that the entire spectrum of the Hessian (and not only the top part) matters. The log-determinant depends not only on the largest eigenvalue, but on the decay rate of \emph{all} the eigenvalues as well. However, for large neural networks, the Hessian is almost inevitably singular, and so its log-determinant is undefined \cite{zhang2021modern,wei2022deep}. Our presented form of $S$ has no such issues, and its relation to the Hessian is well studied \cite{SHS2023,singh2021analytic,liao2021hessian}. This representation of sharpness should prove valuable in further explorations of the correlation between the eigenspectra and test performance as seen in heavy-tailed self-regularization theory \cite{mahoney2019traditional,martin2020heavy,martin2021predicting} and other linearized analyses \cite{bartlett2020benign,RKQW2023,agrawal2022alpha}.

Finally, we remark that in view of the vast literature investigating implicit regularization of stochastic optimizers \cite{neyshabur2017implicit,bietti2019kernel,smith2020origin}, the form of the regularizer $R$ for neural network interpolators is a fertile ground for future research. %

\section*{Acknowledgments}
FR was partially supported by the Australian Research Council through an Industrial Transformation Training Centre for Information Resilience (IC200100022).

\clearpage

\appendix

\section{Appendix}
\label{sxn:appendix}

\begin{proof}[Proof of Theorem \ref{thm:PAC}]

Let $J(\theta) = DF(\theta) DF(\theta)^\top \in \mathbb{R}^{mn\times mn}$ denote the empirical NTK. We start from the core PAC-Bayes bound (\ref{eq:PACBayesCore})
\[
\mathbb{E}_{\rho}L(\theta)\leq\mathbb{E}_{\rho}L_{n}(\theta)+\frac{\log \mathbb{E}_{\mathcal{D}} \zeta_n}{n}+\frac{\kl(\rho\Vert\pi)}{n}+\frac{\log(1/\delta)}{n},
\]
where dummy variables in the expectations have been dropped for brevity. First, since $L_n(\theta) = 0$ on $\mathcal{M}$, $\mathbb{E}_\rho L_n(\theta) = \mathcal{O}(\gamma)$. Next, taking $\pi(\theta) \propto \exp(-\frac1\tau R(\theta))$, applying Laplace's method twice, 
\begin{align*}
\zeta_{n}&=\int_\Theta e^{n(L(\theta)-L_{n}(\theta))}\pi(\theta)\dd\theta=\frac{\int_\Theta e^{n(L(\theta)-L_{n}(\theta))}e^{-\frac{1}{\tau}R(\theta)}\dd\theta}{\int_\Theta e^{-\frac{1}{\tau}R(\theta)}\dd\theta}\\&=e^{n(L(\theta_{0})-L_{n}(\theta_{0}))}+\mathcal{O}(\tau),
\end{align*}
and so $\log\mathbb{E}_{\mathcal{D}}\zeta_{n}=\frac{1}{2}nP+\mathcal{O}(\tau)$. 
Observe that
\begin{align*}
\kl(\rho\Vert\pi) &= \int_{\mathbb{R}^d} \log\left(\frac{\rho(\theta)}{\pi(\theta)}\right)\rho(\theta)\dd\theta\\
&= \frac{1}{\mathcal{Z}_\gamma} \int_{\mathbb{R}^d} \log\left(\frac{\pi(\theta)e^{-\frac{n}{\gamma}L_{n}(\theta)}}{\pi(\theta)\mathcal{Z}_{\gamma}}\right)\pi(\theta)e^{-\frac{n}{\gamma}L_{n}(\theta)}\dd\theta\\
&= \frac{1}{\mathcal{Z}_{\gamma}}\int_{\mathbb{R}^d}\left(-\log\mathcal{Z}_{\gamma}-\frac{n}{\gamma}L_{n}(\theta)\right)\pi(\theta)e^{-\frac{n}{\gamma}L_{n}(\theta)}\dd\theta\\
&=-\log\mathcal{Z}_{\gamma}-\frac{n}{\gamma}\mathbb{E}_{\rho}L_{n}(\theta)\\
&\leq-\log\mathcal{Z}_{\gamma}.
\end{align*}
From the proof of \cite[Theorem 1]{hodgkinson2023interpolating},
\[
-\log \mathcal{Z}_\gamma = \frac{1}{\tau} [R(\theta^\star) - R(\theta_0)] + \frac{mn}{2} \log (\pi\tau) + \frac{S+K}{2}+ \mathcal{O}(\tau). 
\]
In line with \cite{hodgkinson2023interpolating}, choosing $\tau = \frac{2}{mn} [R(\theta^\star) - R(\theta_0)]$, since $R(\theta^\star) = \mathcal{O}(1)$, $\tau = \mathcal{O}(n^{-1})$ and 
\[
-\log \mathcal{Z}_\gamma = \frac{mn}{2}\left(1 + \log\frac{2\pi}{mn} + \log [R(\theta^\star) - R(\theta_0)]\right) + \frac{S+K}{2} + \mathcal{O}(n^{-1}). 
\]
Altogether,
\begin{align*}
\mathbb{E}_\rho L(\theta) &\leq \frac{P}{2} -\frac{\log \mathcal{Z}_\gamma}{n} + \frac{\log(1/\delta)}{n} + \mathcal{O}(\tau) + \mathcal{O}(\gamma)\\
&\leq \frac{P}{2} + \frac{m}{2}\left(1 + \log\frac{2\pi}{mn}\right) + \frac{m}{2}\log[R(\theta^\star)-R(\theta_0)] \\&\qquad+ \frac{S+K}{2n} + \frac{\log(1/\delta)}{n} + \mathcal{O}(n^{-2}) + \mathcal{O}(\gamma).
\end{align*}

\end{proof}

\end{document}